\documentclass[twoside]{article}

\usepackage[accepted]{arxiv}

\usepackage[round]{natbib}

\usepackage{amsmath}
\usepackage{amssymb}
\usepackage{epsfig}
\usepackage{stmaryrd}
\usepackage{subcaption} 

\newtheorem{theorem}{Theorem}

\newtheorem{remark}[theorem]{Remark}

\newtheorem{corollary}[theorem]{Corollary}

\newtheorem{lemma}[theorem]{Lemma}

\newtheorem{proposition}[theorem]{Proposition}
\newtheorem{assumption}[theorem]{Assumption}

\newtheorem{exam}{Example}

\newenvironment{proof}[1][Proof]{\textbf{#1.} }{\ \rule{0.5em}{0.5em}}

\def\E{\mathbb E}

\def\PROB{{\mathbb P}}

\def\I{{\mathbb I}}
\def\pr{\PROB}

\def\argmax{\mathop{\rm arg\, max}}
\def\argmin{\mathop{\rm arg\, min}}

\def\blackslug{\hbox{\hskip 1pt \vrule width 4pt height 8pt depth 1.5pt
\hskip 1pt}}
\def\qed{\quad\blackslug\lower 8.5pt\null\par}

\newcommand{\X}{{\cal X}}

\usepackage{comment}
\usepackage{mathtools}
\usepackage{booktabs}
\usepackage{multirow}
\usepackage{hyperref}
\usepackage{cleveref}

\newcommand{\p}[0]{\mathbb{P}}
\DeclarePairedDelimiter\abs{\lvert}{\rvert}\DeclarePairedDelimiter\norm{\lVert}{\rVert}

\makeatletter
\let\oldabs\abs
\def\abs{\@ifstar{\oldabs}{\oldabs*}}
\let\oldnorm\norm
\def\norm{\@ifstar{\oldnorm}{\oldnorm*}}
\makeatother

\usepackage{color}

\begin{document}

\twocolumn[

\aistatstitle{A Multiclass Classification Approach to Label Ranking}

\aistatsauthor{Stephan Cl\'emen\c{c}on \And Robin Vogel}

\aistatsaddress{LTCI, T\'el\'ecom Paris, \\Institut Polytechnique de Paris
	    \And IDEMIA/LTCI, T\'el\'ecom Paris \\Institut Polytechnique de Paris 
} ]

\begin{abstract}
     In multiclass classification, the goal is to learn how to predict a random label $Y$, valued in $\mathcal{Y}=\{1,\; \ldots,\; K \}$ with $K\geq 3$, based upon observing a r.v. $X$, taking its values in $\mathbb{R}^q$ with $q\geq 1$ say, by means of a classification rule $g:\mathbb{R}^q\to \mathcal{Y}$ with minimum probability of error $\mathbb{P}\{Y\neq g(X) \}$. However, in a wide variety of situations, the task targeted may be more ambitious, consisting in sorting all the possible label values $y$ that may be assigned to $X$ by decreasing order of the posterior probability $\eta_y(X)=\mathbb{P}\{Y=y \mid X \}$. This article is devoted to the analysis of this statistical learning problem, halfway between multiclass classification and posterior probability estimation (regression) and referred to as \textit{label ranking} here. We highlight the fact that it can be viewed as a specific variant of \textit{ranking median regression} (RMR), where, rather than observing a random permutation $\Sigma$ assigned to the input vector $X$ and drawn from a Bradley-Terry-Luce-Plackett model with conditional preference vector $(\eta_1(X),\; \ldots,\; \eta_K(X))$, the sole information available for training a label ranking rule is the label $Y$ ranked on top, namely $\Sigma^{-1}(1)$. Inspired by recent results in RMR, we prove that under appropriate noise conditions, the One-Versus-One (OVO) approach to multiclassification yields, as a by-product, an optimal ranking of the labels with overwhelming probability. Beyond theoretical guarantees, the relevance of the approach to label ranking promoted in this article is supported by experimental results.
 \end{abstract}

\section{INTRODUCTION}\label{sec:intro}
In the standard formulation of the multiclass classification problem, $(X,Y)$ is a random pair defined on a probability space $(\Omega,\; \mathcal{F},\; \mathbb{P})$ with unknown joint probability distribution $P$, where $Y$ is a label valued in $\mathcal{Y}=\{1,\; \ldots,\; K  \}$ with $K\geq 3$ and the r.v. $X$ takes its values in a possibly high-dimensional Euclidean space, say $\mathbb{R}^q$ with $q\geq 1$, and models some input information that is expected to be useful to predict the output variable $Y$. The objective pursued is to build from training data $\mathcal{D}=\{(X_1,Y_1),\; \ldots,\; (X_n,Y_n)\}$, supposed to be independent copies of the generic pair $(X,Y)$, a (measurable) classifier $g:\mathbb{R}^q\rightarrow \mathcal{Y}$ that nearly minimizes the risk of misclassification
\begin{equation}\label{eq:risk}
L(g)=\mathbb{P}\{Y\neq g(X)\}.
\end{equation}
Let $\eta(x)=(\eta_1(x),\; \ldots,\; \eta_K(x))$ be the vector of posterior probabilities: $\eta_k(x)=\mathbb{P}\{Y=k \mid X=x  \}$, for $x\in \mathbb{R}^q$ and $k\in\{1,\; \ldots,\; K  \}$. For simplicity, we assume here that the distribution of the r.v. $\eta(X)$ is continuous, so that the $\eta_k(X)$'s are pairwise distinct with probability one.
It is well-known that the minimum risk is attained by the Bayes classifier 
$$
g^*(x)=\argmax_{k\in\{1,\; \ldots,\; K  \}}\eta_k(x),
$$
and is equal to 
$$
L^*=L(g^*)=1-\mathbb{E}\left[ \max_{1\leq k\leq K}\eta_{k}(X) \right].
$$
As the distribution $P$ is unknown, a classifier must be built from the training dataset and from the perspective of statistical learning theory, the Empirical Risk Minimization (ERM) paradigm encourages us to replace the risk \eqref{eq:risk} by a statistical estimate $\widehat{L}_n(g)$, typically the empirical version $(1/n)\sum_{i=1}^n\mathbb{I}\{Y_i\neq g(X_i) \}$ denoting by $\mathbb{I}\{\mathcal{E} \}$ the indicator function of any event $\mathcal{E}$, and consider solutions $\widehat{g}_n$ of the optimization problem
\begin{equation}\label{erm}
\min_{g\in \mathcal{G}}\widehat{L}_n(g),
\end{equation}
where the infimum is taken over a class $\mathcal{G}$ of classifier candidates, with controlled complexity (\textit{e.g.} of finite {\sc VC} dimension), though supposed rich enough to yield a small bias error $\inf_{g\in\mathcal{G}}L(g)-L^*$, \textit{i.e.} to include a reasonable approximation of the Bayes classifier $g^*$.
Theoretical results assessing the statistical performance of empirical risk minimizers are very well documented in the literature, see \textit{e.g.} \cite{DGL96}, and a wide collection of algorithmic approaches has been designed in order to solve possibly smoothed/convexified and/or penalized versions of the minimization problem \eqref{erm}.  Denoting by $\mathfrak{S}_K$ the symmetric group of order $K$ (\textit{i.e.} the group of permutations of $\{1,\; \ldots,\; K\}$), another natural statistical learning goal in this setup, halfway between multiclass classification and estimation of the posterior probability function $\eta(x)$ and referred to as \textit{label ranking} throughout the article, is to learn, from the training data $\mathcal{D}$, a \textit{ranking rule} $s$, \textit{i.e.} a measurable mapping $s:\mathbb{R}^q \to \mathfrak{S}_K$, such that the permutation $s(X)$ sorts, with 'high probability', all possible label values $k$ in $\mathcal{Y}$ by decreasing order of the posterior probability $\eta_k(X)$, that is to say in the same order as the permutation $\sigma^*_X$ defined by: $\forall x\in \mathbb{R}^q$,
\begin{equation}\label{eq:opt}
    \eta_{\sigma^{*-1}_x(1)}>\eta_{\sigma^{*-1}_x(2)}>\ldots>\eta_{\sigma^{*-1}_x(K)}.
\end{equation}
Equipped with this notation, observe that $g^*(x)=\sigma_x^{*-1}(1)$ for all $x\in\mathbb{R}^q$. Given a loss function $d:\mathfrak{S}_K\times \mathfrak{S}_K \to \mathbb{R}_+$ (\textit{i.e.} a symmetric measurable mapping s.t. $d(\sigma,\; \sigma)=0$ for all $\sigma \in \mathfrak{S}_K$), one may formulate label ranking as the problem of finding a ranking rule $s$ which minimizes the \textit{ranking risk}
\begin{equation}\label{eq:risk_rank}
\mathcal{R}(s)\overset{def}{=}\mathbb{E}\left[d\left(s(X),\sigma^*_X \right)  \right].
\end{equation}
Except when $K=2$ and in the case when the loss function $d$ considered only measures the capacity of the ranking rule to recover the label that is ranked first, that is to say when $d(\sigma,\; \sigma')=\mathbb{I}\{\sigma^{-1}(1)\neq \sigma'^{-1}(1) \}$ (in this case, $\mathcal{R}(s)=\mathbb{P}\{g^*(X)\neq s(X)^{-1}(1) \}$), the nature of the label ranking problem significantly differs from that of multiclass classification. There is no natural empirical counterpart of the risk \eqref{eq:risk_rank} based on the observations $\mathcal{D}$, which makes the ERM strategy inapplicable in a straightforward fashion. It is the goal of the present paper to show that the label ranking problem can be solved, under appropriate noise conditions, by means of the \textit{One-Versus-One} (OVO) approach to multiclass classification. The learning strategy proposed is directly inspired from recent advances in \textit{consensus ranking} and \textit{ranking median regression} (RMR), see \cite{CKS17} and \cite{CKS18}. In the RMR setup, assigned to the input random vector $X$, one considers an output r.v. $\Sigma$ that takes its values in the group $\mathfrak{S}_K$ (in recommending systems, $\Sigma$ may represent the preferences over a set of items indexed by $k\in\{1,\; \ldots,\; K\}$ of a given user, whose profile is described by the features $X$). The goal is to find a ranking rule $s$ that minimizes $\mathbb{E}[d(s(X),\; \Sigma)]$, that is to say, for any $x\in\mathbb{R}^q$, a \textit{consensus/median ranking} $s(x)\in \mathfrak{S}_K$ related to the conditional distribution of $\Sigma$ given $X=x$ w.r.t. the metric $d(.,\; .)$.
In this paper, by means of a coupling technique we show that the label ranking problem stated above can be viewed as a variant of RMR where the output ranking is very partially observed in the training stage, through the label ranked first solely. Based on this analogy, the main result of the article shows that the OVO method permits to recover the optimal label ranking with high probability, provided that noise conditions are fulfilled for all binary classification subproblems. Incidentally, the analysis carried out provides statistical guarantees in the form of (possibly fast) learning rate bounds for the OVO approach to multiclass classification under the hypotheses stipulated. The theoretical results established in this article are also empirically confirmed by various numerical experiments.

The paper is organized as follows. In section \ref{sec:prel}, the OVO methodology for multiclass classification is recalled at length, together with recent results in RMR. The main results of the article are stated in section \ref{sec:main}: principally, a coupling result connecting label ranking to RMR and statistical guarantees for the OVO approach to label ranking in the form of nonasymptotic probability bounds. Numerical experiments are displayed in section \ref{sec:num}, while some concluding remarks are collected in section \ref{sec:concl}. The proofs are deferred to the Appendix section.

\section{PRELIMINARIES}\label{sec:prel}

As a first go, we recall the OVO approach for defining a multiclass classifier from binary classifiers. Basic hypotheses and results related to Ranking Median Regression (RMR) are next briefly described.

\subsection{From Binary to Multiclass Classification}\label{subsec:binary}
A classifier $g$ is entirely characterized by the collection of subsets of the feature space $\mathcal{X}$: $(S_g(1),\; \ldots,\; S_g(K))$, where $S_g(k)=\{x\in \mathcal{X}:\; g(x)=k  \}$ for $k\in\{1,\; \ldots,\; K  \}$. Observe that the $S_k$'s are pairwise disjoint and their union is equal to $\mathbb{R}^q$. Hence, they form a partition of $\mathbb{R}^q$, except that it may happen that a certain subset $S_k(g)$ is empty, \textit{i.e.} a certain label $k$ is never predicted by $g$. 

{\bf The OVO approach.} Partitioning the feature space $\mathbb{R}^q$ in more than two subsets may lead to practical difficulties and certain learning algorithms such as Support Vector Machines (SVM's) are originally tailored to the binary situation (\textit{i.e.} to the case $K=2$). In this case, a natural way of extending such algorithms, usually referred to as the 'One-Versus-One' approach to multi-class classification is to run it $K(K-1)/2$ times, for each binary subproblem, see \textit{e.g.} \cite{HT98}, \cite{MM98}, \cite{ASS00}, \cite{Furn02} or \cite{WLW04}: for any $1\leq k<l\leq K$, based on the fraction of the training data with labels in $\{k,\; l\}$ only,
$$
\mathcal{D}_{k,l}=\left\{ (X_i,Y_i):\; Y_i\in\{k,l  \},\; i=1,\; \ldots,\; n    \right\},
$$ 
the algorithm outputs a classification rule $g_{k,l}:\mathbb{R}^q\rightarrow \{-1,\; +1 \}$ with risk 
$$
L_{k,l}(g_{k,l})\overset{def}{=}\mathbb{P}\{ Y_{k,l}\neq g_{k,l}(X)\mid Y\in\{k,l \}\},
$$ where $Y_{k,l}=\mathbb{I}\{Y=l \}-\mathbb{I}\{Y=k \}$, as small as possible and combine, for any possible input value $x\in \mathbb{R}^q$, the binary predictions $g_{k,l}(x)$ so as to produce a multi-class classifier $\bar{g}:\mathbb{R}^q\to \{1,\; \ldots,\; K\}$ with minimum risk $L(\bar{g})$. A possible fashion of combining the results of the $K(K-1)/2$ 'duels' is to take as predicted label which has won the largest number of duels (and stipulate a rule for breaking possible ties).
The rationale behind this OVO approach lies in the fact that
\begin{equation}\label{eq:1vs1_opt}
g^*(x)=\argmax_{k\in\{1,\; \ldots,\; K  \}}N^*_k(x),
\end{equation}
where, for all $(k,x)\in\{1,\; \ldots,\; K\}\times \mathbb{R}^q$, $N^*_k(x)$
denotes the number of duels won by label $k$ with optimal/Bayes
classifiers for all binary subproblems, namely
$$
N^*_k(x)=\sum_{l<k}\mathbb{I}\{ g_{l,k}^*(x)=+1\}+ \sum_{k<l}\mathbb{I}\{ g_{k,l}^*(x)=-1\} ,
$$
where $g^*_{l,m}(x)=2\mathbb{I}\{\eta_{m}(x)/(\eta_m(x)+\eta_l(x))>1/2 \}-1$ is the minimizer of the risk $L_{l,m}$ for $l<m$. The proof is straightforward. Indeed, it suffices to observe that, for all $i\in\{1,\; \ldots,\; K\}$,
$
N^*_{\sigma^*_x(i)}=K-i.
$
\begin{remark}{\sc (One-Versus-All)} An alternative to the OVO approach in order to reduce multiclass classification to binary subproblems and apply the SVM methodology consists in comparing each class to all of the others in $K$ two-class
duels. A test point is classified as follows: the signed
distances from each of the $K$ separating hyperplanes are computed, the winner being simply the
class corresponding to the largest signed distance. However, other rules have been proposed in \cite{Vapnik98} and in \cite{WestonWatkins99}. 
\end{remark}

\noindent {\bf Label Ranking.}  As underlined in the Introduction section, rather than learning to predict the likeliest label given $X$, it may also be desirable to rank all possible labels according to their conditional likelihood. The goal is then to recover the permutation $\sigma^*_X$ defined through \eqref{eq:opt}.
Practically, this boils down to build a predictive rule $s(x)$ from the training data $(X_1,Y_1),\; \ldots,\; (X_n,Y_n)$ that maps $\mathbb{R}^q$ to $\mathfrak{S}_K$ and minimizes the ranking risk \eqref{eq:risk_rank},
where $d(.,\; .)$ is an appropriate loss function defined on $\mathfrak{S}_K\times \mathfrak{S}_K$. For instance, one may consider $\mathbb{I}\{\sigma\neq \sigma'\}$ or the Hamming distance $\sum_{k=1}^K\mathbb{I}\{\sigma(k)\neq \sigma'(k)\}$ to measure the dissimilarity between two permutations $\sigma$ and $\sigma'$ in $\mathfrak{S}_K$. Classic metrics on $\mathfrak{S}_K$ (see \cite{DH98}) also provide natural choices for the loss function, including
\begin{itemize}
\item the Kendall $\tau$ distance: $\forall (\sigma,\; \sigma')\in \mathfrak{S}_K^2$,
$$
d_{\tau}(\sigma,\; \sigma')=\sum_{i<j}\mathbb{I}\{(\sigma(i)-\sigma(j))\cdot (\sigma'(i)-\sigma'(j))<0\};
$$
\item the Spearman footrule: $\forall (\sigma,\; \sigma')\in \mathfrak{S}_K^2$,
$$
d_1(\sigma,\; \sigma')=\sum_{i=1}^K\left\vert\sigma(i)-\sigma'(i)\right\vert;
$$
\item the Spearman $\rho$ distance: $\forall (\sigma,\; \sigma')\in \mathfrak{S}_K^2$,
$$
d_{2}(\sigma,\; \sigma')=\sum_{i=1}^K(\sigma(i)-\sigma'(i))^2.
$$
\end{itemize}
As shall be explained below, the label ranking problem can be viewed as a variant of the standard ranking median regression problem.
 
\subsection{Ranking Median Regression}
 This problem of minimizing \eqref{eq:risk_rank} shares some similarity with that referred to as \textit{ranking median regression} in \cite{CKS18}, also called \textit{label ranking} sometimes, see \textit{e.g.} \cite{tsoumakas2009mining} and \cite{vembu2010label}. In this supervised learning problem, the output associated with the input variable $X$ is a random vector $\Sigma$ taking its values in $\mathfrak{S}_K$ (expressing the preferences on a set of items indexed by $k\in\{1,\; \ldots,\; K\}$ of a user with a profile characterized by $X$  drawn at random in a certain statistical population) and the goal pursued is to learn from independent copies $(X_1,\Sigma_1),\; \ldots,\; (X_n,\Sigma_n)$ of the pair $(X,\Sigma)$ a (measurable) ranking rule $s:\mathcal{X}\to \mathfrak{S}_K$ that nearly minimizes 
 \begin{equation}\label{eq:RMR_risk}
     R(s)=\mathbb{E}[d(\Sigma,s(X))].
 \end{equation}
 The name \textit{ranking median regression} arises from the fact that any rule mapping $X$ to a median of $\Sigma$'s conditional distribution given $X$ w.r.t. the metric/loss $d$ (refer to \cite{CKS17} for a statistical learning formulation of the consensus/median ranking problem) is a minimizer of \eqref{eq:RMR_risk}, see Proposition 5 in \cite{CKS18}. In certain situations, the minimizer of \eqref{eq:RMR_risk} is unique and a closed analytic form can be given for the latter, based on the pairwise probabilities: $p_{i,j}(x)=\mathbb{P}\{ \Sigma(i)<\Sigma(j) \mid X=x \}:=1-p_{j,i}(x)$ for $1\leq i<j\leq K$ and $x\in\mathbb{R}^q$.
 \begin{assumption}\label{hyp:SST}
For all $x\in \mathbb{R}^q$, we have: $\forall (i,k,l)\in\{1,\; \ldots,\; K  \}^3$, $p_{i,j}(x)\neq 1/2$ and
 \begin{equation}\label{eq:SST}
p_{i,j}(x)>1/2 \text{ and } p_{j,k}(x)>1/2 \; \Rightarrow \; p_{i,k}(x)>1/2.
 \end{equation}
 \end{assumption}
 Indeed, when choosing the Kendall $\tau$ distance $d_{\tau}$ as loss function,
 it has been shown that, under Assumption \ref{hyp:SST}, referred to as
 \textit{strict stochastic transitivity}, the minimizer of \eqref{eq:RMR_risk}
 is almost-surely unique and given by: $\forall k\in\{1,\; \ldots,\; K\}$,
 with probability one:
 \begin{equation}\label{eq:opt_RMR}
     s_X^*(k)=1+\sum_{l\neq k}\mathbb{I}\{p_{k,l}(X)<1/2 \}.
 \end{equation}
 
\begin{remark}{\sc (Conditional BTLP model)} A Bradley-Terry-Luce-Plackett model for $\Sigma$'s conditional distribution given $X$, $P_{\Sigma\mid X}$, assumes the existence of a hidden preference vector $w(X)=(w_1(X),\; \ldots,\; w_K(X))$, where $w_k(X)>0$ is interpreted as a preference score for item $k$ of a user with profile $X$, see \textit{e.g.} \cite{BT52}, \cite{Luce59} or \cite{Plackett75}. The conditional distribution of $\Sigma^{-1}$ given $X$ can be defined sequentially as follows: $\Sigma^{-1}(1)$ is distributed according to a multinomial distribution of size $1$ with support $\mathbf{S}_1=\{1,\; \ldots,\; K\}$ and parameters $w_k(X)/\sum_l w_l(X)$ and, for $k>1$, $\Sigma^{-1}(k)$ is distributed according to a multinomial distribution of size $1$ with support $\mathbf{S}_k=\mathbf{S}_1\setminus\{\Sigma^{-1}(1),\; \ldots,\; \Sigma^{-1}(k-1)\}$ with parameters $w_l(X)/\sum_{m\in \mathbf{S}_k}w_m(X)$, $l\in \mathbf{S}_k$.  The conditional pairwise probabilities are given by $p_{k,l}(X)=w_k(X)/(w_k(X)+w_l(X))$ and one may easily check that Assumption \ref{hyp:SST} is fulfilled as soon as the $w_k(X)$'s are pairwise distinct with probability one. In this case, $s^*(X)$ is the permutation that sorts the $w_k(X)$'s in decreasing order.
\end{remark} 
 In \cite{CKS18}, certain situations where empirical risk minimizers over classes of ranking rules  fulfilling appropriate complexity assumptions can be proved to achieve fast learning rates (\textit{i.e.} faster than $O_{\mathbb{P}}(1/\sqrt{n})$) have been investigated. More precisely, denoting by ${\rm ess}\inf Z$ the essential infimum of any real valued r.v. $Z$, the following 'noise condition' related to conditional pairwise probabilities was considered.
  \begin{assumption}\label{hyp:noise_RMR} We have:
   \begin{equation}
  H={\rm ess} \inf \min_{i<j}\left\vert p_{i,j}(X)-1/2 \right\vert>0.
   \end{equation}
  \end{assumption}
  Precisely, it is shown in \cite{CKS18} (see Proposition 7 therein) that,
  under Assumptions \ref{hyp:SST}-\ref{hyp:noise_RMR}, minimizers of the
  empirical version of  \eqref{eq:RMR_risk} over a {\sc VC} major class of
  ranking rules with the Kendall $\tau$ distance as loss function achieves a
  learning rate bound of order $O_{\mathbb{P}}(1/n)$ (without the impact of
  model bias). Since $\mathbb{P}_X\{s(X)\neq s^*_X\}\leq (1/H)\times
      (R(s)-R(s^*_.))$ (\textit{cf} Eq. $(13)$ in \cite{CKS18}), a bound for
      the probability that the empirical risk minimizer differs from the
      optimal ranking rule at a random point $X$ can be immediately derived.

\section{LABEL RANKING}\label{sec:main}

We now describe at length the connection between label ranking and RMR and state the main results of the article.

\subsection{Label Ranking as RMR}\label{subsec:link}
The major difference with label ranking in the multi-class classification context lies in the fact that only the partial information $\sigma_X^{*-1}(1)$ is observable in presence of noise, under the form of the random label $Y$ assigned to $X$ ($\sigma_X^{*-1}(1)$ being the mode of $Y$'s conditional distribution given $X$), in order to mimic the optimal rule $\sigma^*_X$.

\begin{lemma}\label{lem}
Let $(X,Y)$ be a random pair on the probability space $(\Omega,\; \mathcal{F}.\; \mathbb{P})$. One may extend the sample space so as to build a random variable $\Sigma$ that takes its values in $\mathfrak{S}_K$ and whose conditional distribution given $X$ is a BTLP model with preference vector $\eta(X)=(\eta_1(X),\; \ldots,\; \eta_K(X))$ such that
\begin{equation} \label{eq:partial}
Y=\Sigma^{-1}(1) \text{ with probability one}.
\end{equation}
\end{lemma}

See the Appendix section for the technical proof. The noteworthy fact that the probabilities related to the optimal pairwise comparisons 
$\mathbb{P}\{g^*_{k,l}(X)=+1\mid Y\in\{k,\; l\}\}=\eta_k(X)/(\eta_k(X)+\eta_l(X))$ are given by a BTLP model has been pointed out in \cite{HT98}. With the notations introduced in Lemma \ref{lem}, we have in addition
\begin{eqnarray*}
\mathbb{P}\left\{ \Sigma(k)<\Sigma(l)\mid X \right\}&=&\eta_k(X)/(\eta_k(X)+\eta_l(X)) , \\
&:=& \eta_{k,l} (X) .
\end{eqnarray*}
Eq. \eqref{eq:partial} can be interpreted as follows: the label ranking problem as defined in subsection \ref{subsec:binary} can be viewed as a specific RMR problem under strict stochastic transitivity (\textit{i.e.} Assumption \ref{hyp:SST} is always fulfilled) with \textit{incomplete observations}
$$
\left(X_1,\; \Sigma_1^{-1}(1)\right),\; \ldots,\; \left(X_n,\; \Sigma_n^{-1}(1))\right).$$
Due to the incomplete character of the training data, one cannot recover the optimal ranking rule $\sigma^*_x$ by minimizing a statistical version of \eqref{eq:RMR_risk} of course. As an alternative, one may attempt to build directly an empirical version of $\sigma^*_X$ based on the explicit form \eqref{eq:opt_RMR}, which only involves pairwise comparisons, in a similar manner as in \cite{CKS17} for consensus ranking. Indeed, in the specific RMR problem under study, Eq. \eqref{eq:opt_RMR} becomes
\begin{equation}\label{eq:Copeland_form}
\sigma^*_X(k)=1+\sum_{l\neq k}\mathbb{I}\{g^*_{k,l}(X)=-1\},
\end{equation}
 for  all $k\in \{1,\; \ldots,\; K\}$. The OVO procedure precisely permits to construct such an empirical version. As shall be shown by the subsequent analysis, in spite of the very partial nature of the statistical information at disposal, the OVO approach permits to recover the optimal RMR rule $\sigma^*_X$ with high probability provided that $(X,\; \Sigma)$ fulfills (a possibly weakened version of) Assumption \ref{hyp:noise_RMR}, combined with classic complexity conditions.
 Using \cite{NIPS2018_8114} or \cite{ECML_PKDD19_LabelRanking}, one can tackle RMR with partial information,
 but lacks theoretical guarantees. 

\begin{remark}{\sc (On the noise condition)} Attention should be paid to the fact that, when applied to the random pair $(X,\Sigma)$ defined in Lemma \ref{lem}, Assumption \ref{hyp:noise_RMR} simply means that the classic Massart's noise condition is fulfilled for every binary classification subproblem, see \cite{MasNed06}.
\end{remark}

\subsection{The OVO Approach to Label Ranking}\label{subsec:OVO}
Let $\mathcal{G}$ be a class of decision rules $g:\mathbb{R}^q\to \{-1,\; +1\}$.
As stated in subsection \ref{subsec:binary}, the OVO approach to multiclass classification is implemented as follows. For all $k<l$, compute a minimizer $\widehat{g}_{k,l}$ of the empirical risk
\begin{equation}\label{eq:emp_risk_bin}
\widehat{L}_{k,l}(g)=\frac{1}{n_k+n_l}\sum_{i:\; Y_i\in\{k,\; l\}}\mathbb{I}\{g(X_i)\neq Y_{k,l,i}\}
\end{equation} 
over class $\mathcal{G}$, with $Y_{k,l,i}=\mathbb{I}\{Y_i=l\}-\mathbb{I}\{Y_i=k\}$ for $i\in\{1,\; \ldots,\; n\}$ and the convention that $0/0=0$. We set $\widehat{g}_{l,k}=-\widehat{g}_{k,l}$ for $k<l$ by convention. Equipped with these $\binom{K}{2}$ classifiers, for any test (\textit{i.e.} input and unlabeled) observation $X$, the $\widehat{g}_{k,l}(X)$'s define a \textit{complete directed graph} $G_X$ with the $K$ labels as vertices: $\forall k<l$, $l\to_X k$ if $\widehat{g}_{k,l}(X)=+1$ and $k\to_X l$ otherwise. The analysis carried out in the next subsection shows that under appropriate noise conditions, with large probability, the random graph $G_X$ is \textit{acyclic}, meaning that the complete binary relation $l\to_X k$ is transitive (\textit{i.e.} $l\to k$ and $k\to_X m$ $\Rightarrow$ $l\to_X m$), in other words that the \textit{scoring function}
\begin{multline}\label{eq:Copeland}
\widehat{s}(X)(k)=1+\sum_{k\neq l}\mathbb{I}\left\{\widehat{g}_{k,l}(X)=-1\right\},\\
=1+\sum_{k\neq l}\mathbb{I}\left\{k \to_X l\right\},
\text{ for } k\in\{1,\; \ldots,\; K\}
\end{multline}
defines a permutation, which, in addition, coincides with $\sigma^*_X$, \textit{cf} Eq. \eqref{eq:Copeland_form}. The equivalence between the transitivity of $\to_X$, the acyclicity of $G_X$ and the membership of $\widehat{s}(X)$ in $\mathfrak{S}_K$ is straightforward, details are left to the reader (see \textit{e.g.} the argument of Theorem 5's proof in \cite{CKS18}). The quantity \eqref{eq:Copeland} can be related to the Copeland score, see \cite{Copeland51}: the score $\widehat{s}(X)(k)$ of label $k$ being equal to $1$ plus the number of duels it has lost, while its Copeland score $C_X(k)$ is its number of victories minus its number of defeats, so that \begin{equation}
\widehat{s}(X)(.) =\left(K+1-C_X(.)\right)/2.
\end{equation} 

\begin{figure}[!h]
	
	\begin{center}
		\fbox{
			\begin{minipage}[t]{7.5cm}
				\begin{center}
					{\sc OVO Approach to Label Ranking}
				\end{center}
				
				{\small
					
					\begin{enumerate}
						\item[] {\bf Inputs.} Class $\mathcal{G}$ of classifier candidates. Training classification dataset $\mathcal{D}=\{(X_1,Y_1),\; \ldots,\, (X_n,Y_n)  \}$. Query point $x\in \mathbb{R}^q$.
						\item ({\sc Binary classifiers.}) For $k<l$, based on $\mathcal{D}_{k,l}=\{(X_i,Y_i):\; Y_i\in\{k,\; l\},\; i=1,\; \ldots,\; n\}$, compute the ERM solution to the binary classification problem:
						$$
						\widehat{g}_{k,l}=\argmin_{g\in \mathcal{G}}\widehat{L}_{k,l}(g).
						$$
						
						\item ({\sc Scoring.}) Compute the  predictions $\widehat{g}_{k,l}(x)$ and the score for the query point $x$:
						$$
						\widehat{s}(x)(k)=1+\sum_{l\neq k}\mathbb{I}\{\widehat{g}_{k,l}(x)=-1 \}.
						$$
						
						\item[] {\bf Output.} Break arbitrarily possible ties in order to get a prediction $	\widehat{\sigma}_{x}$ in $\mathfrak{S}_K$ at $x$ from $\widehat{s}(x)$.
					\end{enumerate}
				}
			\end{minipage}
		}
		\caption{Pseudo-code for 'OVO label ranking'}\label{fig:trpseudoknn}
	\end{center}
\end{figure}

When $G_X$ is not transitive, or equivalently when $\widehat{s}(X)\notin
\mathfrak{S}_K$, one may build a ranking $\widehat{\sigma}_X$ from the scoring
function \eqref{eq:Copeland} by breaking ties in an arbitrary fashion, as
proposed below for simplicity. Alternatives could be considered of course.
The issue of building a ranking/permutation of the labels in $\{1,\; \ldots,\; K
\}$ from \eqref{eq:Copeland} can be connected with the \textit{feedback set}
problem for directed graphs, see \textit{e.g.} \cite{DETT99}: for a directed
graph, a minimal feedback arcset is a set of edges of smallest cardinality such
that a directed acyclic graph is obtained when reversing the edges in it. Refer
to \textit{e.g.} \cite{Festa1999} for algorithms.

\subsection{Statistical Guarantees for Label Ranking}
It is the purpose of the subsequent analysis to show that, provided that the conditions listed below
are fulfilled, the ranking rule $\sigma^*_X$ can be fully recovered through the OVO approach previously described with high probability. We denote by $\mu$ the marginal distribution of the input variable $X$, by $\mu_k$ the conditional distribution of $X$ given $Y=k$ and set $p_k=\mathbb{P}\{Y=k\}$ for $k\in\{1,\; \ldots,\; K  \}$.
\begin{assumption}\label{hyp:noise} There exists $\alpha\in[0,1]$ and $B>0$ such that: for all $k<l$ and $t\geq 0$,
$$
\mathbb{P}\left\{\vert 2 \eta_{k,l}(X)- 1 \vert <t  \right\}
\leq Bt^{\frac{\alpha}{1-\alpha}}.
$$
\end{assumption}
\begin{assumption}\label{hyp:comp}
The class $\mathcal{G}$ is of finite {\sc VC} dimension $V<+\infty$.
\end{assumption}
\begin{assumption}\label{hyp:min} There exists a constant $\varepsilon>0$, s.t. for all $k\neq l$ in $\{1,\; \ldots,\; K  \}$ and $x\in \X$, $ \eta_k(x)+\eta_l(x)>\varepsilon$.
\end{assumption}
Assumption \ref{hyp:noise} means that Assumption \ref{hyp:noise_RMR} is satisfied by the random pair $(X,\Sigma)$ defined in Lemma \ref{lem} in the case $\alpha=1$ (notice incidentally that it is void when $\alpha=0$) and reduces to the classic Mammen-Tsybakov noise condition in the binary case $K=2$, see \cite{Mammen}. The following result provides nonasymptotic bounds for the ranking risk $\mathcal{R}_P(\widehat{\sigma}_X)$ of the OVO ranking rule in the case where the loss function is $\mathbb{I}\{\sigma\neq \sigma'  \}$, \textit{i.e.} for the probability of error. Extension to any other loss function $d(.,\; .)$ is straightforward, insofar as we obviously have $d(\sigma^*_X,\; \widehat{\sigma}_X)\leq \max_{(\sigma,\sigma')\in \mathfrak{S}_K^2}d(\sigma,\sigma')\times \mathbb{I}\{\widehat{\sigma}_X \neq \sigma^*_X \}$ with probability one.

\begin{theorem}\label{thm:main}
    Suppose that Assumptions \ref{hyp:noise}-\ref{hyp:min} are fulfilled. Then, for all $\delta\in (0,1)$, we have with probability (\textit{w.p.}) at least $1-\delta$: $\forall n\geq n_0(\delta, \alpha,\epsilon, B, V)$,
\begin{multline*}
\mathbb{P}\left\{ \widehat{\sigma}_X\neq \sigma^*_X \mid \mathcal{D}\right\}\leq \\ 
\frac{\beta}{\varepsilon}\left\{\binom{K}{2} r_n^{\alpha}\left( \frac{\delta}{\binom{K}{2}} \right) +\sum_{k<l} 2  \left( \inf_{g\in \mathcal{G}}L_{k,l}(g)-L_{k,l}^*  \right)^{\alpha}      \right\},
\end{multline*}
where $X$ denotes a r.v. drawn from $\mu$, independent from the training data $\mathcal{D}$, $L^*_{k,l}=L_{k,l}(g^*_{k,l})$, $\beta = \beta(\alpha, B)$ and with $h:=h(B, a, \epsilon)$,
\begin{multline*}
r_n(\delta)=
2  \left( 1/(nh) \right)^{\frac{1}{2-\alpha}} \times \\
	    \left[ \left(64 C^2 V\log n\right)^{\frac{1}{2-\alpha}}
	    + \left( 32 \log(2/\delta) \right)^{\frac{1}{2-\alpha}} \right].
\end{multline*}
\end{theorem}

Refer to the Appendix section for the technical proof.
\begin{remark} {\sc (On the noise condition (bis))} We point out that the results of this paper can be straightforwardly extended to the situation where the noise exponent $\alpha$ may vary depending on the binary subproblem considered. For the sake of simplicity only, here we restrict the analysis to the homogeneous setup described by Assumption \ref{hyp:noise}.
\end{remark} Hence, for the RMR problem related to the partially observed BTLP model detailed in subsection \ref{subsec:link}, the rate bound achieved by the OVO ranking rule in Theorem \ref{thm:main} is of order $O_{\mathbb{P}}(n^{-\alpha/(2-\alpha)})$, ignoring the bias term and the logarithmic factors. In the case $\alpha=1$, it is exactly the same rate as that attained by minimizers of the ranking risk in the standard RMR setup, as stated in Proposition 7 in \cite{CKS17}.
Whereas situations where the OVO multi-class classification may possibly lead to 'inconsistencies' (\textit{i.e.} where the binary relationship $\to_X$ is not transitive) have been exhibited many times in the literature, no probability bound for the excess of classification risk of the general OVO classifier, built from ERM applied to all binary subproblems, is documented to the best our knowledge. Hence, attention should be paid to the fact that, as a by-product of the argument of Theorem \ref{thm:main}'s proof, generalization bounds for the OVO classifier 
$$
\bar{g}(X)\overset{def}{=}\widehat{\sigma}_X^{-1}(1).
$$
can be established, as stated in Corollary \ref{cor} below. More generally, the
statistical performance of the label ranking rule $\widehat{\sigma}_x$ produced
by the method described in subsection \ref{subsec:OVO} can be assessed for
other risks. For instance, rather than just comparing the true label $Y$
assigned to $X$ to the label $\widehat{\sigma}_X^{-1}(1)$ ranked first, as in
OVO classification approach, one could consider $\ell_k(Y,\;
\widehat{\sigma}_X)$, with
$
\ell_k(y,\; \sigma)=\mathbb{I}\{y\notin \{\sigma^{-1}(1),\; \ldots,\; \sigma^{-1}(k)\}  \}$
for all $(y,\sigma)\in\{1,\; \ldots,\; K \}\times\mathfrak{S}_K$, equal to $1$ when $Y$ does not appear in the top $k$ list and to $0$ otherwise, where $k$ is fixed in $\{1,\; \ldots,\; K\}$. For any ranking rule $s:\mathbb{R}^q\to \mathfrak{S}_K$, the corresponding risk is then 
\begin{equation}\label{eq:topk_risk}
W_k(s)=\mathbb{E}[\ell_k(Y,\; s(X))  ].
\end{equation}
Set $W_k^*=\min_s W_k(s)$, where the minimum is taken over the set of all possible ranking rules $s$. As shown in the Appendix section, the argument leading to Theorem 10 can be adapted to prove a rate bound for the risk excess of the OVO ranking rule $\sigma^*_.:x\in\mathbb{R}^q\mapsto \sigma^*_x$.

\begin{proposition}\label{prop}
Let $k\in\{1,\; \ldots,\; K\}$ be fixed. Then:
$$
W_k^*=W_k(\sigma^*_.).
$$
Suppose in addition that Assumptions \ref{hyp:noise}-\ref{hyp:min} are fulfilled. Then, for all $\delta\in (0,1)$, we have \textit{w.p.} $\ge 1-\delta$: $\forall n\geq 1$,
\begin{multline*}
W_k(\widehat{\sigma}_.)-W_k^*\leq \frac{\beta}{\varepsilon} \binom{K}{k}k(K-k)\times\\ 
    \left( r_n^{\alpha}\left( \frac{\delta}{\binom{K}{2}} \right)+ 2 \cdot\max_{m\neq l}\left(   \inf_{g\in \mathcal{G}}L_{l,m}(g)-L_{l,m}^* \right)^{\alpha} \right).
\end{multline*}
\end{proposition}
Since we have $W_1(s)=L(s(.)^{-1}(1))$ for any label ranking rule $s(x)$, in the case $k=1$ the result above provides a generalization bound for the excess of misclassification risk of the OVO classifier $\bar{g}(x)=\widehat{\sigma}_x^{-1}(1)$.

\begin{corollary}\label{cor}
    Suppose that Assumptions \ref{hyp:noise}-\ref{hyp:min} are fulfilled. Then, for all $\delta\in (0,1)$, we have \textit{w.p.} $\ge 1-\delta$: 
$\forall n\geq n_0(\delta, \alpha,\epsilon, B, V)$,
\begin{multline*}
    L(\bar{g})-L^*\leq \frac{\beta}{\varepsilon} K(K-1)\times\\ \left( r_n^{\alpha} \left (\frac{\delta}{\binom{K}{2}} \right )+2 \cdot \max_{k\neq l}\left(   \inf_{g\in \mathcal{G}}L_{k,l}(g)-L_{k,l}^* \right)^{\alpha} \right).
\end{multline*}
\end{corollary}

\section{EXPERIMENTAL RESULTS}\label{sec:num}
This section first illustrates the results of Theorem \ref{thm:main}  using simulated datasets, of which distributions
satisfy Assumption \ref{hyp:noise},  for certain values of the noise parameter $\alpha$, highlighting the impact/relevance of this condition. In the experiments based on real data next displayed, the OVO
approach to top-$k$ classification, \textit{cf} Eq. \eqref{eq:topk_risk}, is shown to surpass rankings relying on the scores output by multiclass classification algorithms.
Due to space limitations, details and comments are postponed to the Supplementary Material.

\noindent {\bf Synthetic data.}
In this toy illustrative example, we consider $\X = [0,1]$, $K=8$ and learn a
simple decision stump, \textit{i.e.} a function of the form $ x \mapsto
2\I\left \{ (x-s)\epsilon \ge 0 \right \} - 1$ where $s, \epsilon$ are unknown
parameters. A representation of the $\eta_k$'s for all $k \in \{ 1, \dots, K
\}$ as well as the expected Kendall $\tau$ distance of OVO label ranking models
for different values of $n$ are given in Fig. \ref{fig:learning-in-n-dist}.
For each value of $n$, the boxplot is computed using $100$ independent trials, representing different learning rates, for $\alpha = 0.2$ and $\alpha = 0.8$ namely.

\begin{figure}[h!] 
    \centering 
\includegraphics[width=0.52\linewidth]{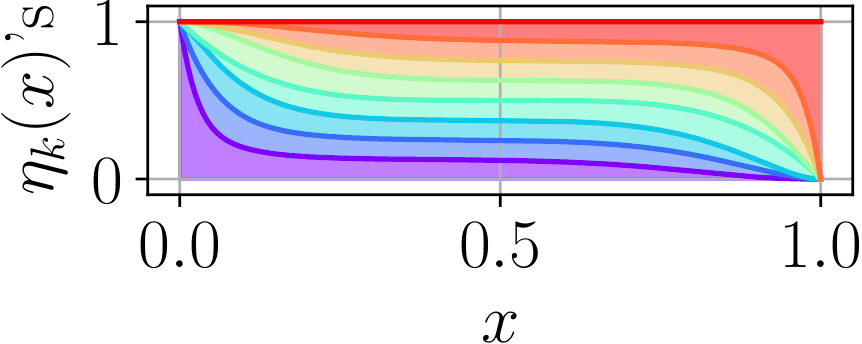}
    \includegraphics[width=0.46\linewidth]{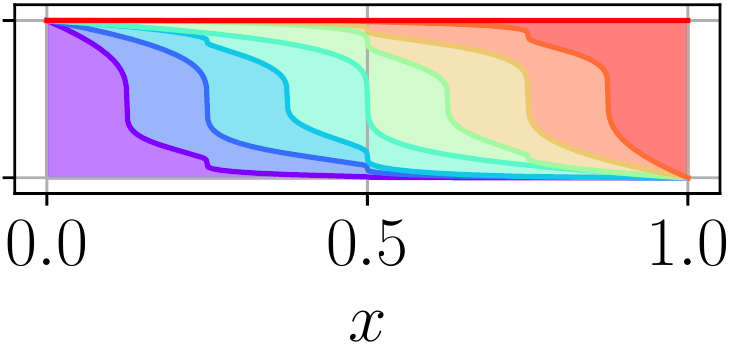}
    \begin{subfigure}{\linewidth}
      \centering
      \includegraphics[width=0.8\linewidth]{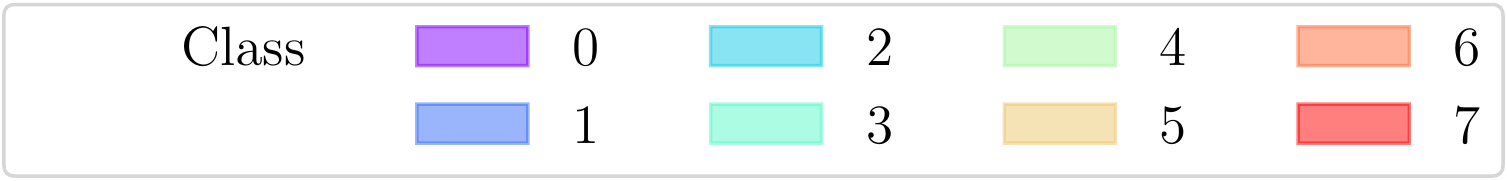}
      \label{fig:eta-legend}
    \end{subfigure}
    \medskip

    \begin{subfigure}{.53\linewidth}
      \centering
      \includegraphics[width=\linewidth]{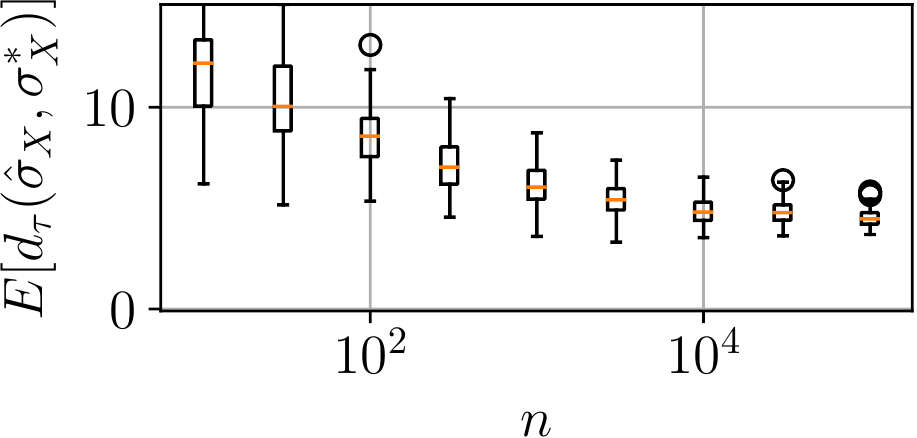}
      \caption{$\alpha=0.2$.}
      \label{fig:learning-in-n-noisy}
    \end{subfigure}\begin{subfigure}{.45\linewidth}
      \centering
      \includegraphics[width=\linewidth]{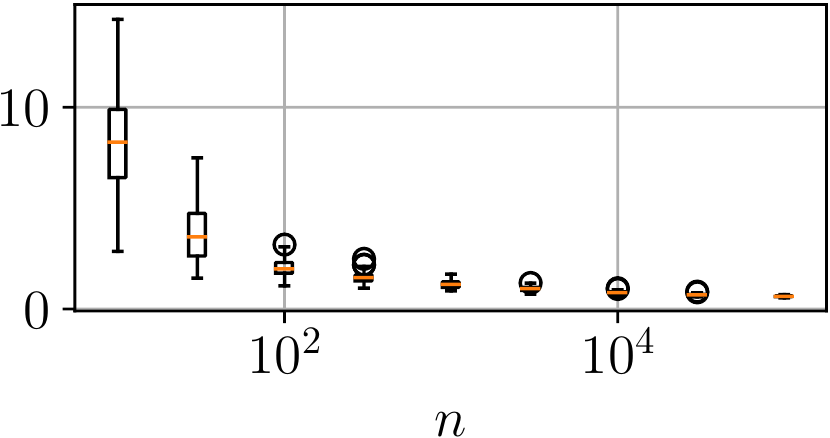}
      \caption{$\alpha=0.8$}
      \label{fig:learning-in-n-separ}
    \end{subfigure}
    \caption{Probability of each class on $[0,1]$ for $\alpha \in\{0.2, 0.8\}$ and boxplot of
100 independent estimations of $\E\left[ d_\tau(\hat{\sigma}_X, \sigma^*_X) \right]$
	as a function of $n$.} 
\label{fig:learning-in-n-dist}
\end{figure}

\noindent {\bf Real data.} Regarding top-$k$ performance, for two popular datasets,
MNIST and fashion-MNIST, the OVO label ranking approach is benchmarked against
the rankings based on the probability estimates related to a multiclass
logistic regression in Table \ref{tab:real-data-exps}.

\begin{table}[h]
      \centering
      { \footnotesize
	\caption{Top-$k$ performance. The time to fit the model is given by the last column.}
      \label{tab:real-data-exps}
	\begin{tabular}[h]{ccccc}
	    \toprule
	    Dataset & Model & Top-1  & Top-5 & Fit time\\
	\cmidrule{1-5}
	\multirow{2}{*}{MNIST} &  LogReg &  0.924 &  0.995 & 50 min \\
	&  OVO &  \textbf{0.943} &  \textbf{0.997} & 40 min\\
\cmidrule{1-5}
	Fashion &  LogReg & 0.857 & 0.997 & 35 min\\
	-MNIST &  OVO & \textbf{0.863} & \textbf{0.997} & 60 min\\
	\bottomrule
	\end{tabular}
  }
\end{table}
 
\section{CONCLUSION}\label{sec:concl}
In this paper, a statistical problem halfway between multiclass classification and posterior probability estimation, referred to as \textit{label ranking} here, is considered. The goal is to design a method to rank, for any test observation $X$, all the labels $y$ that can be possibly assigned to it by decreasing order of magnitude of the (unknown) posterior probability $\mathbb{P}\{Y=y\mid X  \}$. Formulated as a specific ranking median regression problem with incomplete observations, this problem is shown to have a solution that takes the form of a Copeland score, involving pairwise comparisons only. Based on this crucial observation, it is proved that the OVO procedure for multiclass classification permits to build, from training classification/labelled data, the optimal ranking with high probability, under appropriate hypotheses. This is also empirically supported by numerical experiments. Remarkably, the analysis carried out here incidentally provides a rate bound for the OVO classifier.

\section*{APPENDIX - TECHNICAL DETAILS}
\subsection*{Proof of Lemma \ref{lem}}
As a first go, define $\Sigma^{-1}(1)$ as $Y$. Next, given $X$ and $\Sigma^{-1}(1)=Y$, draw $\Sigma'$ as a BTLP model on the set $\mathcal{I}=\{1,\; \ldots,\; K\}\setminus\{\Sigma^{-1}(1)\}$ with preference parameters $\eta_k(X)$, $k\in \mathcal{I}$. For all $r\in\{1,\; \ldots,\; K-1\}$, set $\Sigma^{-1}(r+1)=\Sigma'^{-1}(r)$ and invert the permutation $(\Sigma^{-1}(1),\; \ldots,\; \Sigma^{-1}(K))$ to get a random permutation $\Sigma$ with the desired properties. \subsection*{Proof of Theorem \ref{thm:main}}
Fix $\delta\in{0,1}$ and let $1\leq k<l\leq K$. Assumption \ref{hyp:noise} implies that the Mammen-Tsybakov noise condition is fulfilled for the binary classification problem related to the pair $(X,Y)$ given that $Y\in\{k,\; l\}$. When Assumptions \ref{hyp:comp}-\ref{hyp:min} are also satisfied, a possibly fast rate bound for the risk excess of the empirical risk minimizer $\widehat{g}_{k,l}$ can be established, as stated in the following lemma.

\begin{lemma}\label{lem2} Suppose that Assumptions \ref{hyp:noise}-\ref{hyp:min} are fulfilled. Let $1\leq k<l\leq K$. Then, for all $\delta\in (0,1)$, we have \textit{w.p.} $\ge 1-\delta$: $\forall n\geq 1$,
\begin{equation}
L_{k,l}(\widehat{g}_{k,l})-L_{k,l}^*\leq 2\left(\inf_{g\in \mathcal{G}}L_{k,l}(g)-L_{k,l}^*  \right)+r_n(\delta),
\end{equation}
where, for all 
$n\geq n_0(\delta, \alpha,\epsilon, B, V)$ and $\delta\in (0,1)$,
\begin{multline*}
r_n(\delta)=
2  \left( 1/(nh) \right)^{\frac{1}{2-\alpha}} \times \\
	    \left[ \left(64 C^2 V\log n\right)^{\frac{1}{2-\alpha}}
	    + \left( 32 \log(2/\delta) \right)^{\frac{1}{2-\alpha}} \right].
\end{multline*}
\end{lemma}
\begin{proof}
The result is a slight variant of that proved in \cite{BBM05} (see therein), the sole difference lying in the fact that the empirical risk (and, consequently, its minimizer as well) is built from a random number of training observations (\textit{i.e.} those with labels in $\{k,\; l\}$).
Note that $h =\epsilon^{3-2\alpha} (1-\alpha)^{1-\alpha} \alpha^\alpha/ B^{1-\alpha}$.
Details are given in the Supplementary Material. \end{proof}

Observe that the probabilities appearing in this proof are conditional probabilities given the training sample $\mathcal{D}$ and, as a consequence, must be considered as random variables. However, to simplify notations, we omit to write the conditioning w.r.t. $\mathcal{D}$ explicitly.
Notice first that Assumption \ref{hyp:noise} implies that, 
with $\beta = B^{1-\alpha}/(\left( 1 - \alpha \right)^{1-\alpha} \alpha^\alpha)$,
\begin{multline}\label{eq:control1}
\mathbb{P}_X\left\{ \widehat{g}_{k,l}(X) \neq g^*_{k,l}(X) \mid Y\in\{k,\; l \}  \right\}\leq \\ \beta \left( L_{k,l}(\widehat{g}_{k,l})-L_{k,l}^* \right)^{\alpha},
\end{multline}
with probability one. Observe in addition that
\begin{multline}
\mathbb{P}_X\left\{ \widehat{g}_{k,l}(X) \neq g^*_{k,l}(X) \mid Y\in\{k,\; l \}  \right\}=\\
\mathbb{E}_X\left[\frac{d\mu_{k,l}}{d\mu}(X)\times \mathbb{I}\left\{ \widehat{g}_{k,l}(X) \neq g^*_{k,l}(X) \right\} \right],
\end{multline}
denoting by $\mu_{k,l}=(p_k\mu_k+p_l\mu_l)/(p_k+p_l)$ the conditional distribution of $X$ given that $Y\in\{k,\; l\}$. Under Assumption \ref{hyp:min}, we almost-surely have:
$$
\frac{d\mu_{k,l}}{d\mu}(X)\geq \frac{\varepsilon}{p_k + p_l} \geq \varepsilon.
$$
Hence, from \eqref{eq:control1} and Lemma \ref{lem2}, we get that
\begin{multline}\label{bound_bin}
\frac{\varepsilon}{\beta}\mathbb{P}_X\left\{ \widehat{g}_{k,l}(X) \neq g^*_{k,l}(X)  \right\}\leq \\
\left( L_{k,l}(\widehat{g}_{k,l})-L_{k,l}^* \right)^{\alpha}\leq  2 \left(\inf_{g\in \mathcal{G}}L_{k,l}(g)-L_{k,l}^*  \right)^{\alpha}+r_n^{\alpha}(\delta),
\end{multline}
using Minkowski's inequality.
Since
$$
\bigcap_{k<l}\left\{ \widehat{g}_{k,l}(X)= g^*_{k,l}(X)   \right\}\subset \left\{  \sigma^*_{X}=\widehat{\sigma}_X   \right\},
$$ with probability one, combining the bound above with the union bound, for all $\delta\in (0,1)$, \textit{w.p.} $\ge 1-\delta$:
\begin{multline*}
\mathbb{P}_X\left\{\sigma^*_{X}\ne\widehat{\sigma}_X\right\}
\leq \sum_{k<l}\mathbb{P}_X\left\{ \widehat{g}_{k,l}(X) \neq g^*_{k,l}(X)  \right\}\leq \\
\frac{\beta}{\varepsilon} \left\{\binom{K}{2} r_n^{\alpha}
    \left( \frac{\delta}{\binom{K}{2}} \right)
+ \sum_{k<l}2 \left( \inf_{g\in \mathcal{G}}L_{k,l}(g)-L_{k,l}^*  \right)^{\alpha}      \right\}.
\end{multline*}

 \subsection*{Proof of Proposition \ref{prop}} 
Let us first show that $ W_k^*=W_k(\sigma^*_.).$ 

For any ranking rule $s$ and all $x\in\mathbb{R}^q$, we define
$$
{\rm Top}_k(s(x))=\{s(X)^{-1}(1),\; \ldots,\;  s(X)^{-1}(k)\},
$$
and also set ${\rm Top}_k^*(x)={\rm Top}_k(\sigma^*_x)$.
Indeed, for any ranking rule $s$, we can write 
$$
W_k(s)=\mathbb{E}\left[ \mathbb{E}\left[ \ell_k(Y,\; s(X))  \mid X\right] \right],
$$ and we almost-surely have
\begin{multline}\label{eq:cond_exp}
\mathbb{E}\left[ \ell_k(Y,\; s(X))  \mid X\right]=\\
\sum_{l=1}^K \eta_l(X)\mathbb{I}\{ l\notin {\rm Top}_k(s(X)) \}.
\end{multline}
As $\sigma^*_x$ is defined through \eqref{eq:opt}, one easily sees that the quantity \eqref{eq:cond_exp} is minimum for any ranking rule $s(x)$ s.t.
\begin{equation}\label{eq:opt2}
{\rm Top}_k(s(X))={\rm Top}_k^*(X).
\end{equation}
Hence, the collection of optimal ranking rules regarding the risk \eqref{eq:topk_risk} coincides with the set of ranking rules such that \eqref{eq:opt2} holds true with probability one.
Observe that, with probability one,
\begin{multline*}
\mathbb{I}\left\{Y\notin {\rm Top}_k(s(X))  \right\}- \mathbb{I}\left\{Y\notin {\rm Top}_k^*(X)  \right\}\leq \\
\mathbb{I}\left\{ {\rm Top}_k^*(X)\neq {\rm Top}_k(s(X))  \right\},
\end{multline*}
for any ranking rule $s(x)$, so that
$$
W_k(s)-W^*_k\leq \mathbb{P}_{X}\left\{ {\rm Top}_k(s(X)) \neq {\rm Top}_k^*(X) \right\}.
$$
In addition, notice that
\begin{multline*}
W_k(\widehat{\sigma}_X)-W_k^*\leq \mathbb{P}_{X}\left\{ {\rm Top}_k^*(X)\neq {\rm Top}_k(\widehat{\sigma}_X)  \right\}=\\
\sum_{\mathcal{L}\subset \mathcal{Y}:\; \# \mathcal{L}=k} \mathbb{P}_{X} \left\{{\rm Top}_k^*(X)=\mathcal{L},\;  {\rm Top}_k^*(X)\neq {\rm Top}_k(\widehat{\sigma}_X)   \right\},\\
\leq \sum_{\mathcal{L}\subset \mathcal{Y}:\; \# \mathcal{L}=k}\sum_{l\in\mathcal{L},\; m\notin \mathcal{L}}\mathbb{P}_{X}\left\{  \widehat{g}_{l,m}(X)\neq g^*_{l,m}(X)  \right\}, \qquad \\
\leq \frac{\beta}{\varepsilon}
\binom{K}{k}k(K-k)\times \qquad \qquad \qquad \qquad \qquad \qquad \\ 
 \left( r_n^{\alpha}\left( \frac{\delta}{\binom{K}{2}} \right) + 2 \cdot \max_{m\neq l}\left(   \inf_{g\in \mathcal{G}}L_{l,m}(g)-L_{l,m}^* \right)^{\alpha} \right),
\end{multline*}
using \eqref{bound_bin}.

\bibliographystyle{abbrvnat}
\setcitestyle{authoryear,open={((},close={))}}

\onecolumn
\section*{APPENDIX - SUPPLEMENTARY MATERIAL}
\subsection{Detailed proof of Lemma \ref{lem2}}

As pointed out, the result is a slight variant of that proved in 
\cite[pages 342-346]{Boucheron2005}. The sole difference lies in the fact that
fact that the empirical risk (and, consequently, its minimizer as well) is
built from a random number of training
observations (\textit{i.e.} those with labels in $\{k,\; l\}$).
Here, we detail the proof for completion.

The derivation of fast learning speeds for general classes of functions relies
on a sensible use of Talagrand's inequality that exploits the upper bound
on the variance of the loss provided by the noise condition, combined with
convergence bounds on Rademacher averages, see \cite{BBM05}.

To begin with, we define classes of functions that mirror the ones used by \cite{Boucheron2005}.
Those are specifically introduced for the problem of associating elements of
the sample to the label $k$ or $l$, with $k<l, (k,l) \in \{ 1, \dots, K \}^2$
any pair of labels.
Given a label $y\in\mathcal{Y}$, it corresponds to solving binary classification for
$y_{k,l} = \I\{ y = k\} - \I\{ y= l \}$ for all of the concerned instances, \textit{i.e.}
those with labels $k$ or $l$.
For each binary classifier $g$ in $\mathcal{G}$, we introduce the cost function $c_{k,l}$ and 
the proportion of concerned instances $h_{k,l}$, such that,
for all $x,y \in \X \times \{ 1, \dots, K\}$,
\begin{align*}
    c_{k,l}(x, y) = \I\{ g(x) \ne y_{k,l}, y \in \{ k,l \} \} \quad 
    \text{and} \quad h_{k,l}(y) = \I\{y \in \{k,l\}\}.
\end{align*}
Denote by $\mathcal{F}_{k,l}$ the regret of each function $c_{k,l}$, formally:
\begin{align*}
    \mathcal{F}_{k,l} := \left\{ f_{k,l}: x,y \mapsto \I\{y \in \{k,l\}\} \cdot
    \left( c_{k,l}(x, y) - \I\{ g_{k,l}^*(x) \ne y_{k,l} \} \right) \mid g \in
    \mathcal{G} \right\}.
\end{align*}
With $P$ as the expectation over $X,Y$ and $P_n$ as the empirical measure, 
one can rewrite the risk $L_{k,l}$ and empirical risk $\widehat{L}_{k,l}$ as:
\begin{align*}
    L_{k,l}(g) = \frac{P c_{k,l}}{P h_{k,l}} \quad \text{and} \quad 
    \widehat{L}_{k,l}(g) = \frac{P_n c_{k,l}}{P_n h_{k,l}}.
\end{align*}
Unlike $c_{k,l}$ the empirical mean $P_nh_{k,l}$ does not depend on an element
of $g\in\mathcal{G}$, thus minimizing $\widehat{L}_{k,l}$ is the same problem
as minimizing $P_n c_{k,l}$.
The rest of the proof consists in using \cite[section 5.3.5 therein]{Boucheron2005}
to derive an upper bound of $Pf$, with $f \in \mathcal{F}_{k,l}$. Talagrand's inequality
is useful because of an upper-bound on the variance of the elements in $\mathcal{F}_{k,l}$.

Assumption \ref{hyp:noise} induces a control on the variance of the elements of
$\mathcal{F}_{k,l}$. \cite[page 202 therein]{Bousquet2004} reviewed equivalent formulations
of the noise assumption, in the case of binary classification. One of those
formulations is similar to the following equation:
\begin{align}\label{eq:noise-implied}
    \p \left\{ g(X) \ne g^*_{k,l}(X) \right\} \le 
    \beta ( L_{k,l}(g) - L_{k,l}^* )^\alpha,
\end{align}
where $\beta = \frac{ B^{1-\alpha} }{\epsilon (1-\alpha) ^{1-\alpha} \alpha^\alpha}$.
The proof is the same as that of \cite[page 202 therein]{Bousquet2004}, 
but is followed by Assumption \ref{hyp:min}, see \cref{eq_noise_cond} for more details.

Set $\beta_0 = \beta (p_k + p_l)^\alpha$, Equation \eqref{eq:noise-implied} implies, 
for any $f \in \mathcal{F}_{k,l}$, that,
with $T(f) = \sqrt{\beta_0} \cdot (Pf)^{\alpha/2}$:
\begin{align}\label{var-control}
    \text{Var}(f) \le  \p \left \{ g(X) \ne g^*_{k,l}(X) \right\}
    \le  \beta (L_{k,l}(g) - L_{k,l}^* )^\alpha = 
    \beta (p_k + p_l)^\alpha \cdot (Pf)^\alpha 
    = T^2(f).
\end{align}
The function $T(f)$ controls the variance of the elements in $\mathcal{F}_{k,l}$, and
is used to reweights its instances before applying Talagrand's inequality.

The complexity of the proposal family of functions is controlled using the notion of Rademacher average, as in \cite{Boucheron2005}.
Let $\mathcal{F}$ be a class of functions, its Rademacher average $R_n(\mathcal{F})$ is defined as:
\begin{align*}
  R_n(\mathcal{F}) := 
  \E_{\sigma} \sup_{f\in\mathcal{F}} \frac{1}{n}\abs{\sum_{i=1}^n \sigma_i f(X_i, Y_i)}.
\end{align*}
Introduce $\mathcal{F}_{k,l}^*$ as the star-hull of $\mathcal{F}_{k,l}$,
\textit{i.e.} $\mathcal{F}_{k,l}^* = \{ \alpha f : \alpha \in [0,1], f \in
\mathcal{F}_{k,l} \}$,
we define two functions that characterize the properties of the problem of interest, and are required to apply \cite[Theorem 5.8 therein]{Boucheron2005}:
\begin{align}\label{w-def}
    w(r) = \sup_{f \in \mathcal{F}_{k,l}^* : Pf \le r } T(f) \qquad \text{ and } \qquad
    \psi(r) = \E R_n\{ f\in \mathcal{F}_{k,l}^* : T(f) \le r \}.
\end{align}
Finally \cite[Theorem 5.8 therein]{Boucheron2005} implies that, for all $\delta >0$, with $r^*_0(\delta)$ the solution of:
\begin{align}\label{r-star-def}
    r = 4 \psi (w(r)) + 2 w(r) \sqrt{ \frac{2 \log \frac{2}{\delta} }{n}} + \frac{16 \log \frac{2}{\delta} }{3n},
\end{align}
we have that, with probability at least $1-\delta$,
\begin{align*}
    L_{k,l}( \hat{g}_{k,l}) - L_{k,l}^* \le  2 \left( \inf_{g\in \mathcal{G}} L_{k,l}(g) - L_{k,l}^* \right) + \frac{r_0^*(\delta)}{p_k + p_l}.
\end{align*}

Now, we can conclude by combining this result with properties of $w$ and $\psi$,
that originate from the noise assumption and the control on the complexity of $\mathcal{G}$,
respectively.
Assumption \ref{hyp:comp} states that the proposal class $\mathcal{G}$ is of VC-dimension $V$.
Permanence properties of VC-classes of functions, 
see \cite[section 2.6.5 therein]{VanderVaart1996}, 
imply that $\mathcal{F}_{k,l}$ is also VC.
It follows from \cite{BBM05} that:
\begin{align*}
    \psi(r) \le Cr\sqrt{\frac{V}{n} \log n}.
\end{align*}
Plugging this result into \Cref{r-star-def} gives:
\begin{align*}
    r^*_0(\delta) \le \frac{2w(r^*_0(\delta))}{\sqrt{n}} \left[ 2 C\sqrt{V \log n},
    +  \sqrt{ 2 \log \frac{2}{\delta}} \right] + \frac{16 \log \frac{2}{\delta} }{3n}.
\end{align*}
Combining it with the definition of $w$ in \eqref{w-def} and the control on the
variance laid forth in \eqref{var-control} yields:
\begin{align}\label{eq:inter-00}
    r^*_0(\delta) \le \left[r^*_0(\delta)\right] ^{\alpha/2} \frac{2 \sqrt{\beta_0}}{\sqrt{n}} 
    \left[ 2 C\sqrt{V \log n}
    +  \sqrt{ 2 \log \frac{2}{\delta}} \right] + \frac{16 \log \frac{2}{\delta} }{3n}.
\end{align}

Equation \eqref{eq:inter-00} is a variational inequality, and an upper bound on the
solution can be derived directly from \cite[Lemma 2 therein]{Zou2009}. It writes:
\begin{align*}
    r^*_0(\delta) \le \max \left\{ \left( \frac{16 \beta_0}{n} \right)^{\frac{1}{2-\alpha}}
    \left[ 2 C \sqrt{V\log n} + \sqrt{2 \log(2/\delta)} \right]^{\frac{2}{2-\alpha}},
    \frac{32 \log (2/\delta)}{3n}
\right\}.
\end{align*}
Using the convexity of $x\mapsto x^{\frac{2}{2-\alpha}}$, the right-hand side 
of the above inequality can be upper-bounded, which leads to:
\begin{align}\label{eq:pre-conclu}
    r^*_0(\delta) \le 2 \cdot \max \left\{ \left( \frac{16 \beta_0}{n} \right)^{\frac{1}{2-\alpha}}
	    \left[ \left(4 C^2 V\log n\right)^{\frac{1}{2-\alpha}}
	    + \left( 2 \log(2/\delta) \right)^{\frac{1}{2-\alpha}} \right],
    \frac{32 \log (2/\delta)}{3n}
    \right\}.
\end{align}
Assumption \ref{hyp:min} implies that $(p_k+p_l)^{-1} \le 1/\epsilon$. 
Introducing $r^*(\delta) = (p_k+p_l)^{-1} r_0^*(\delta)$, Equation \eqref{eq:pre-conclu} combined with the definition of $\beta_0$ implies:
\begin{align}\label{eq:pre-conclu-2}
    r^*(\delta) \le 2 \cdot \max \left\{ \left( \frac{16 \beta}{\epsilon^{2-2\alpha} n} \right)^{\frac{1}{2-\alpha}}
	    \left[ \left(4 C^2 V\log n\right)^{\frac{1}{2-\alpha}}
	    + \left( 2 \log(2/\delta) \right)^{\frac{1}{2-\alpha}} \right],
    \frac{32 \log (2/\delta)}{3 \epsilon n}
    \right\}.
\end{align}
Introduce $n_0(\delta, \alpha, \epsilon, B, V)$ as the lowest $n$ such that the first term in the maximum in Equation \eqref{eq:pre-conclu-2} dominates the second term,
it satisfies:
\begin{align}\label{eq:ineq-n1}
    n_0^{\frac{1-\alpha}{2-\alpha}} \left[ \left( 4 C^2 V \log(n_0) \right)^{\frac{1}{2-\alpha}}
    + \left( 2 \log (2/\delta) \right)^{\frac{1}{2-\alpha}} \right] \ge \frac{32 \log(2/\delta)}{3 \left[ 16 \beta \epsilon^\alpha \right]^{\frac{1}{2 - \alpha}}}.
\end{align}
and so does any $n \ge n_0$.

To conclude, we have proven that for any $\delta \in (0,1)$, for any $n \ge n_0(\delta, \alpha, \epsilon, B, V)$, we have that, with probability greater than $1-\delta$,
\begin{align*}
    L_{k,l}( \hat{g}_{k,l}) - L_{k,l}^* \le  2 \left( \inf_{g\in \mathcal{G}} L_{k,l}(g) - L_{k,l}^* \right) + r^*(\delta),
\end{align*}
with
\begin{align}\label{eq:final-r}
    r^*(\delta) = 2  \left( \frac{16 \beta}{n\epsilon^{2-2\alpha}} \right)^{\frac{1}{2-\alpha}}
	    \left[ \left(4 C^2 V\log n\right)^{\frac{1}{2-\alpha}}
	    + \left( 2 \log(2/\delta) \right)^{\frac{1}{2-\alpha}} \right].
\end{align}

\textbf{A simple upper bound on $n_0(\delta, \alpha, \epsilon, B, V)$:}\\

The right-side terms of Equation \eqref{eq:final-r} are not balanced. Indeed, as soon as $n$ is high, the term
in $\log(n)$ dominates the other. That fact can be exploited to derive a convenient upper bound on 
$n_0(\delta, \alpha, \epsilon, B, V)$, since one cannot directly solve \cref{eq:ineq-n1}.

Assume that $n \ge (2/\delta)^{\frac{1}{2C^2V}}$, then
\begin{align*}
    n^{\frac{1-\alpha}{2-\alpha}} \left[ \left( 4 C^2 V \log(n) \right)^{\frac{1}{2-\alpha}}
    + \left( 2 \log (2/\delta) \right)^{\frac{1}{2-\alpha}} \right] \ge 
    2  \left( 2 n^{1-\alpha} \log (2/\delta) \right)^{\frac{1}{2-\alpha}}.
\end{align*}
However, we have that:
\begin{align*}
    2  \left( 2 n^{1-\alpha} \log (2/\delta) \right)^{\frac{1}{2-\alpha}} \ge 
    \frac{32 \log(2/\delta)}{3 \left[ 16 \beta \epsilon^\alpha \right]^{\frac{1}{2 - \alpha}}},
\end{align*}
if and only if:
\begin{align*}
n \ge  \log(2/\delta) \left( \frac{(16/3)^{2-\alpha}}{32 \beta \epsilon^\alpha} \right)^{\frac{1}{1-\alpha}}.
\end{align*}
Hence, we have proven that:
\begin{align*}
    n_0(\delta, \alpha, \epsilon, B, V) \le \max \left\{  
	(2/\delta)^{\frac{1}{2C^2V}},
	\log(2/\delta) \left( \frac{(16/3)^{2-\alpha}}{32 \beta \epsilon^\alpha} \right)^{\frac{1}{1-\alpha}} 
    \right\}.
\end{align*}

\subsubsection{On the equivalent noise condition}\label{eq_noise_cond}

The proof is almost similar to that of \cite[page 202 therein]{Bousquet2004},
and is recalled here.
The excess loss of a classifier can be written as follows:
\begin{align*}
    L_{k,l}(g) - L^*_{k,l} &= 
\E \left[ \abs{2\eta_{k,l}(X) - 1} 
    \cdot \I\left\{ g(X) \ne g_{k,l}^*(X) \right\} \mid Y \in \left\{ k,l \right\}\right].
\end{align*}
Using Markov's inequality, for all $t \ge 0$:
\begin{align*}
     L_{k,l}(g) - L_{k,l}^*  & \ge 
    t \cdot \E \left[ \I\left\{ g(X) \ne g_{k,l}^*(X) \right\} \cdot
    \I\left\{ \abs{2\eta_{k,l}(X) - 1} \ge t \right\} 
    \mid Y \in \left\{ k,l \right\} \right], \\
  &\ge t \cdot  \pr \left \{\abs{2\eta_{k,l}(X) - 1} \ge t \mid Y \in \left\{ k,l \right\} \right\} \\
  & \qquad \qquad - t \cdot \E \left[ \I\left\{ g(X) = g^*_{k,l}(X) \right\}
  \I\left\{  \abs{2\eta_{k,l}(X) - 1} \ge t \right\} \mid Y \in \left\{ k,l \right\} \right].
\end{align*}
Assumption \ref{hyp:noise} and $\I\left\{  \abs{2\eta_{k,l}(X) - 1} \ge t \right\} \le 1$
imply:
\begin{align*}
     L_{k,l}(g) - L_{k,l}^* &\ge t \cdot (1- Bt^{ \frac{\alpha}{1-\alpha}} ) 
    - t \cdot \pr \left\{ g(X) = g_{k,l}^*(X) \mid Y \in \left\{ k,l \right\} \right\}\\
    &\ge t \cdot \left( \pr \left\{ g(X) \ne g_{k,l}^*(X) \mid Y \in \left\{ k,l \right\} \right\} 
    - Bt^{ \frac{\alpha}{1-\alpha}} \right),
\end{align*}
Choosing
\begin{align*}
    t = \left[ \frac{(1-\alpha) \pr \left\{ g(X) \ne g_{k,l}^*(X) 
    \mid Y \in \left\{ k,l \right\} \right\}}{ B} \right]^{(1-\alpha)/\alpha},
\end{align*}
finally gives:
\begin{align*}
    \pr\left\{ g(X) \ne g_{k,l}^*(X) \mid Y \in \left\{ k,l \right\} \right\}
    \le \frac{B^{1-\alpha} }{(1-\alpha)^{1-\alpha} \alpha^{\alpha} }
  (L_{k,l}(g) - L_{k,l}^* ) ^\alpha.
\end{align*}
Which implies, using Assumption \ref{hyp:min}:
\begin{align*}
    \pr\left\{ g(X) \ne g_{k,l}^*(X) \right\}
    \le \frac{B^{1-\alpha} }{\epsilon (1-\alpha)^{1-\alpha} \alpha^{\alpha} }
  (L_{k,l}(g) - L_{k,l}^* ) ^\alpha.
\end{align*}

\subsection{Experiments on simulated data}
Introduce the function $h_\alpha$, which for any $\alpha\in[0,1]$:
\begin{align*}
    h_\alpha(x) = \frac{1}{2} + \frac{1}{2} \epsilon(x) \abs{2x-1}^{\frac{1-\alpha}{\alpha}},
\end{align*}
where $\epsilon(x) = 2\I\left\{ 2x > 1 \right\} - 1 $. It has good properties
with regard to the Mammen-Tsybakov noise condition introduced in
\cite{Boucheron2005}.
We define a warped version $h_{\alpha,x_0}$ of this function $h_\alpha$ such that:
\begin{align*}
    h_{\alpha,x_0}(x) = 
    \begin{cases}
	h_{\alpha,x_0} = h_\alpha \left( \frac{x}{2x_0} \right)   & \text{if } x < x_0,\\
	h_{\alpha,x_0} = h_\alpha \left( \frac{1}{2} + \frac{x-x_0}{2(1-x_0)} \right)   & \text{if } x \ge x_0.\\ \end{cases}
\end{align*}
We use this function to define the $\eta_k$'s by recursion, and assume that $X \sim \mathcal{U}([0,1])$.
Formally, define a depth parameter $D$, the variable $Y$ belong to $K = 2^{D+1}$ classes.
Let $b_2^{(d)}(k)$ describe the decomposition in base 2 of the value $k$, \textit{i.e.} $k = \sum_{d=0}^D 2^{b_2^{(d)}(k)}$,
we set, with $x_{(d,k)} = \sum_{d=1}^D 2^{-b_2^{(d)}(k)}$: \begin{align*}
\eta_k(x) = \prod_{d=0}^D h_{\alpha, x_{(d,k)}}(x).
\end{align*}
By varying the parameter $\alpha$, one can set the classification problems to be more complicated or more easy.
If $\alpha$ is close to $1$, the problems are very simple. If $\alpha$ is close to $0$, the problems are more arduous.
\begin{figure}[h]
    \centering
    \begin{subfigure}{.39\linewidth}
      \centering
      \includegraphics[width=\linewidth]{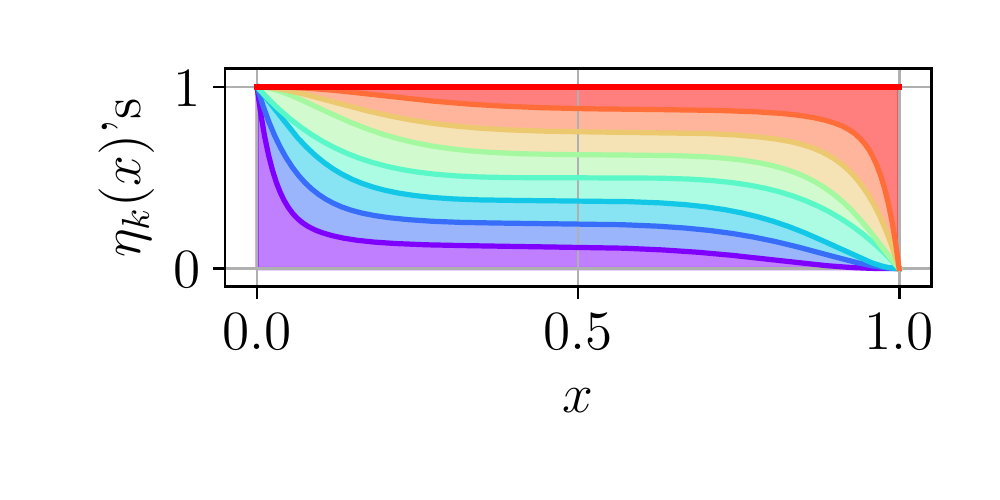}
      \caption{Distribution for $\alpha=0.2$.}
      \label{fig:eta-noisy}
    \end{subfigure}\begin{subfigure}{.39\linewidth}
      \centering
    \includegraphics[width=\linewidth]{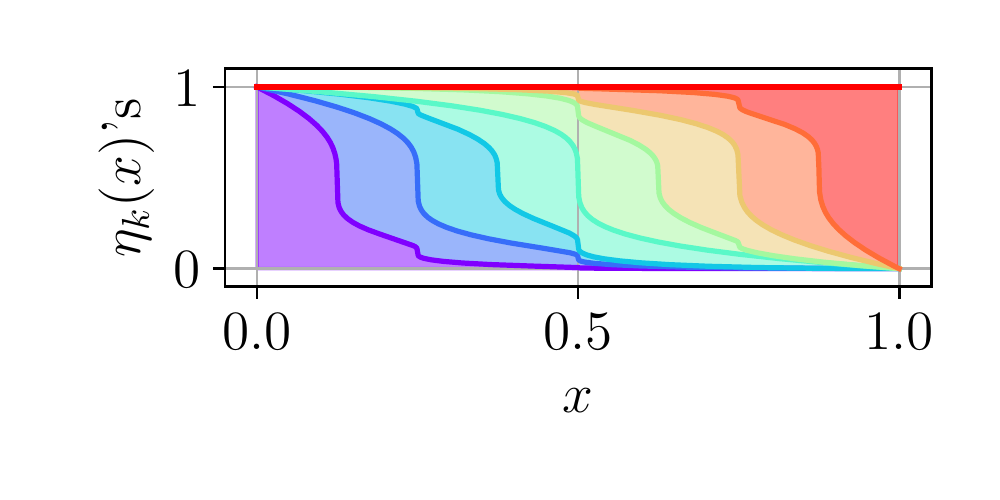}
      \caption{Distribution for $\alpha=0.8$.}
      \label{fig:eta-separ}
    \end{subfigure}

	\bigskip
    \begin{subfigure}{\linewidth}
      \centering
      \includegraphics[width=0.4\linewidth]{figures/legend.pdf}
      \label{fig:eta-legend}
    \end{subfigure}
    \caption{Cumulated histogram of the $\eta_k$'s over $[0,1]$.}
\label{fig:eta-dist}
\end{figure}

To implement the procedure described in \Cref{fig:trpseudoknn}, we learn decision
stumps in $[0,1]$, \textit{i.e.} we optimize over the family of functions
$\mathcal{G} = \left\{ g_{s, \epsilon} \mid s \in [0,1], \epsilon \in \{ -1, +1 \} \right\}$,
where, for any $x\in[0,1]$:
\begin{align*}
    g_{s, \epsilon}(x) = 2\I\left \{ (x-s)\epsilon \ge 0 \right \} - 1.
\end{align*}

\Cref{fig:cycle-simu-supp}, \Cref{fig:loss-simu-supp} 
and \Cref{fig:kendall-simu-supp} 
represent boxplots obtained with 100 independent estimations on 
1000 test points of, respectively,
     the number of cycles in predicted permutations,
     the average miss probability for the problem of predicting
     permutations $\p\{ \hat{\sigma}_X \ne \sigma_X^* \}$,
     and the average Kendall distances between predictions and ground
	    truths $\E\left[ d_\tau(\hat{\sigma}_X, \sigma_X^*) \right]$,
all as a function of the number of learning points $n$ with
\begin{align*}
    n \in \bigcup_{i\in [\![1, 4 ]\!]} \{ 10^i , 3\times 10^i\} \cup \{ 10^5 \} .
\end{align*}

\begin{figure}
    \centering
    \begin{subfigure}{.4\linewidth}
      \centering
      \includegraphics[width=\linewidth]{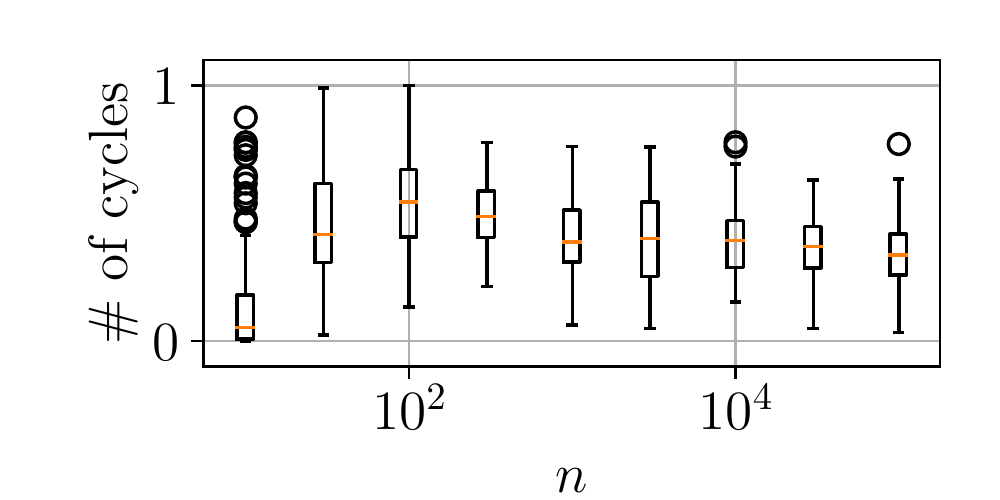}
      \caption{Dynamics for $\alpha=0.2$.}
      \label{fig:learning-in-n-noisy}
    \end{subfigure}\begin{subfigure}{.4\linewidth}
      \centering
      \includegraphics[width=\linewidth]{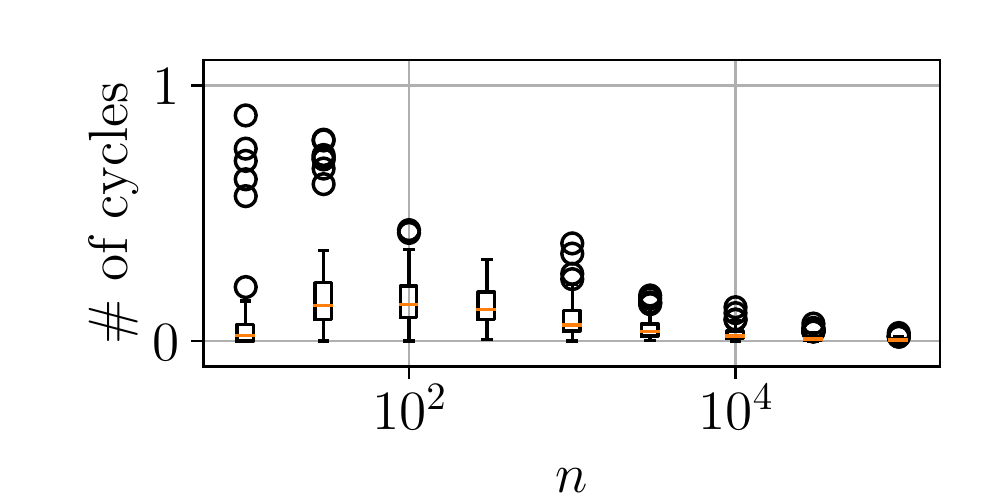}
      \caption{Dynamics for $\alpha=0.8$}
      \label{fig:learning-in-n-separ}
    \end{subfigure}
    \caption{Boxplot of 100 independent estimations of the proportion of
    predictions with cycles as a function of $n$.}
    \label{fig:cycle-simu-supp}
    \bigskip
    \bigskip
    \centering
    \begin{subfigure}{.4\linewidth}
      \centering
      \includegraphics[width=\linewidth]{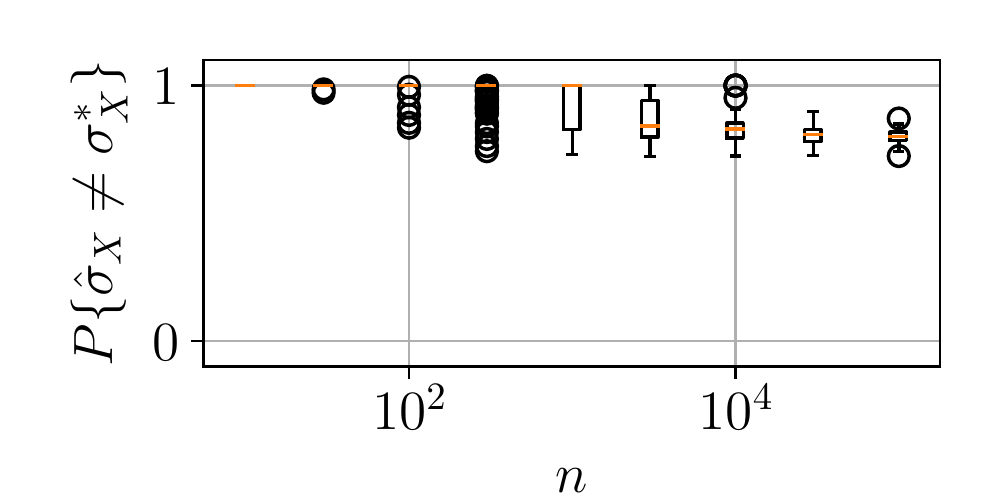}
      \caption{Dynamics for $\alpha=0.2$.}
      \label{fig:learning-in-n-noisy}
    \end{subfigure}\begin{subfigure}{.4\linewidth}
      \centering
      \includegraphics[width=\linewidth]{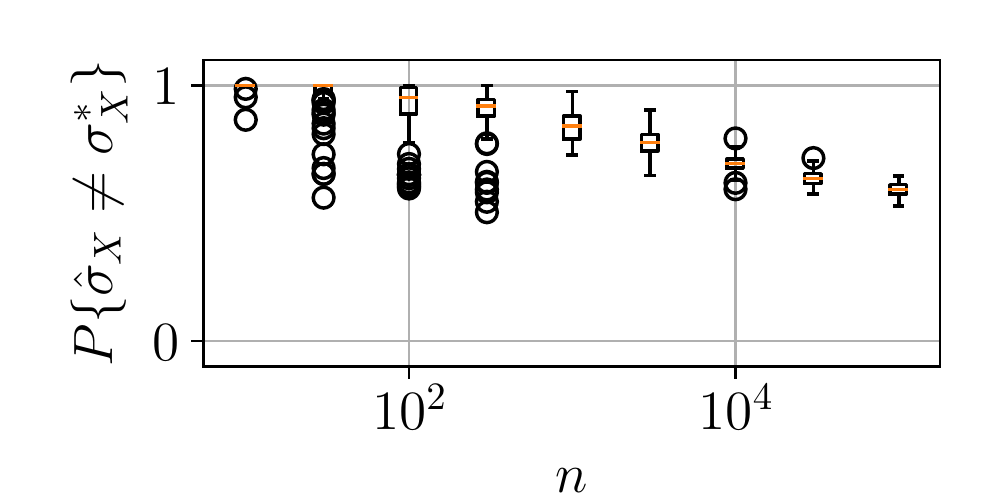}
      \caption{Dynamics for $\alpha=0.8$}
      \label{fig:learning-in-n-separ}
    \end{subfigure}
    \caption{Boxplot of 100 independent estimations of $\p\{ \widehat{\sigma}_X \ne \sigma^*_X\}$
	as function of $n$.}\label{fig:loss-simu-supp}
    \bigskip
    \bigskip
    \begin{subfigure}{.4\linewidth}
      \centering
      \includegraphics[width=\linewidth]{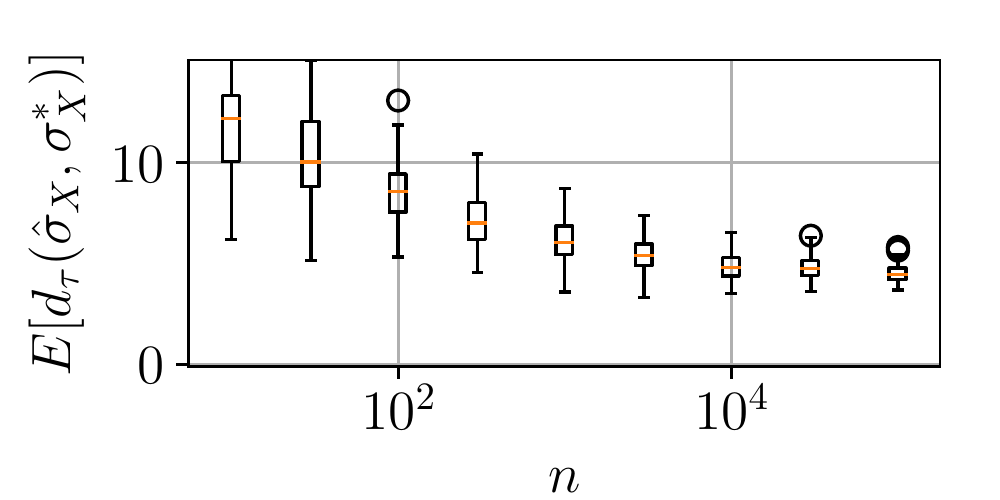}
      \caption{Dynamics for $\alpha=0.2$.}
      \label{fig:learning-in-n-noisy}
    \end{subfigure}\begin{subfigure}{.4\linewidth}
      \centering
      \includegraphics[width=\linewidth]{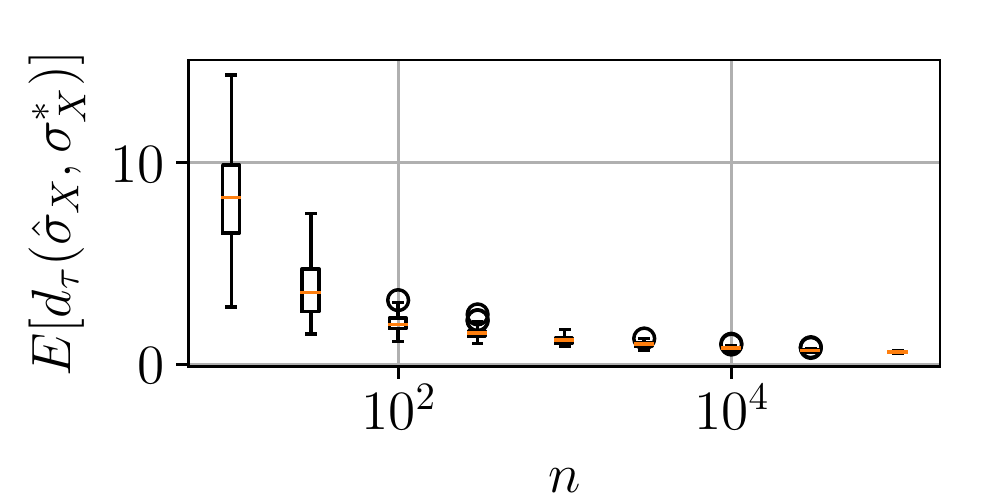}
      \caption{Dynamics for $\alpha=0.8$}
      \label{fig:learning-in-n-separ}
    \end{subfigure}
    \caption{Boxplot of 100 independent estimations of 
	$\E\left[ d_\tau(\hat{\sigma}_X, \sigma_X^*) \right]$ 
    as function of $n$.}\label{fig:kendall-simu-supp}
\end{figure}

One sees that learning is fast when $\alpha$ is close to 1, as expected.
\Cref{fig:kendall-simu-supp} shows that the average Kendall's $\tau$ distance
decreases quickly when $\alpha$ is close to $1$, as does the
the proportion of cycle in predictions, see \Cref{fig:cycle-simu-supp}.
On the other hand, due to the difficulty of predicting a complete permutation,
the influence of the noise parameter on the evolution of the probability of error when $n$
grows is more subtle, see \Cref{fig:loss-simu-supp}.
 
\subsection{Experiments on real data}

The MNIST dataset is composed of $28 \times 28$ grayscale images of digits and labels being
the value of the digits. In this experiment, we learn to predict the value of the
digit between $K=10$ classes corresponding to digits between $0$ and $9$.
The dataset contains $60,000$ images for training and $10,000$ images for
testing, all equally distributed within the classes. This dataset has been
praised for its accessibility, but was recently criticised for being too easy,
which led to the introduction of the dataset Fashion-MNIST, see \cite{fashion-mnist}.
It has the same structure as MNIST, with regard to train and test splits, number
of classes and and image size. It consists in classifying types of clothing apparel, 
\textit{e.g.} dress, coat and sandals, and is harder to classify than MNIST.

Our experiments aim to show that the OVO approach for top-$k$ classification,
\textit{cf} Eq. \eqref{eq:topk_risk}, can surpass rankings relying on
the scores output by multiclass classification algorithms. For that
matter, we evaluated the performances of both approaches using a logistic
regression to solve binary classification in the OVO case and multiclass classification
in the other. For that matter, we relied on the implementations provided by
the \texttt{python} package \texttt{scikit-learn}, specifically the 
\texttt{LogisticRegressionCV} class. The dimensionality of the data was reduced using standard
PCA with enough components to retain 95\% of the variance for both datasets,
which makes for $153$ components for MNIST and $187$ components for Fashion-MNIST.

Results are summarized in \Cref{tab:real-data-exps}. They show that the OVO approach
performs better than the logistic regression for the top-$1$ accuracy, \textit{i.e.}
classification accuracy, as well as for the top-$5$ accuracy. While the OVO approach
requires us to train $K(K-1)/2 = 45$ models, those are trained with less data and
output values. Both approaches end up requiring a similar amount of time to be trained.

\end{document}